\def\nostatisticaltests{1}
\documentclass[letterpaper]{article}
\usepackage{aaai2027}
\usepackage[hyphens]{url}  \usepackage{graphicx}    \usepackage{natbib}  \usepackage{caption}    

\usepackage{algorithm}
\usepackage{algcompatible}

\usepackage{amsmath,amssymb,amsthm}

\graphicspath{{figures/}}

\newtheorem{theorem}{Theorem}
\newtheorem{lemma}{Lemma}
\newtheorem{assumption}{Assumption}

\title{Probabilistic Focal Search: Accelerating Bounded-Suboptimal Search via Lower-Bound Advancement}

\author{
    Minh Vu Duc\textsuperscript{\rm 1}\footnote{Corresponding author:
    \texttt{minhvd@neu.edu.vn}.},
    Trung Le Huu\textsuperscript{\rm 2},
    Hà Minh Hoàng\textsuperscript{\rm 3},
    Trung Thanh Nguyen\textsuperscript{\rm 4},
    Phuong Khanh Nguyen\textsuperscript{\rm 5},
    Huynh Thi Thanh Binh\textsuperscript{\rm 5}
}
\affiliations{
    \textsuperscript{\rm 1}SLSCM Lab, National Economics University\\
    \textsuperscript{\rm 2}University of Warwick\\
    \textsuperscript{\rm 3}DataOpt Lab, National Economics University\\
    \textsuperscript{\rm 4}CADA Lab, National Economics University\\
    \textsuperscript{\rm 5}Hanoi University of Science and Technology
}
\begin{document}

\maketitle

\begin{abstract}
Bounded-suboptimal search seeks a solution within a factor $w$ of optimal while
reducing search effort. Focal Search (FS) uses heuristic guidance within FOCAL,
the frontier nodes eligible under the threshold $w f_{\min}$, but its
deterministic policy may leave $f_{\min}$ unchanged for many expansions. We
introduce Probabilistic Focal Search (PFS), which follows the FS guided choice
with probability $p$ and expands a minimum-$f$ OPEN node with probability
$1-p$. The latter branch encourages the lower bound to advance, enlarging
FOCAL and admitting nodes that may lead to feasible solutions. By balancing
guidance and lower-bound advancement, this mechanism can reduce time to a
bounded solution when progress is limited by delayed FOCAL admission. As a
secondary transfer experiment, we apply
the same scheduler to Dynamic Potential Search, yielding Probabilistic Dynamic
Potential Search (PDPS). We benchmark PFS against FS on N-Puzzle, Pancake
Sorting, and the
Traveling Salesperson Problem (TSP), and evaluate its anytime extension on the
Generalized Covering TSP (GCTSP), using multiple $w$ and $p$ values. Across
these benchmarks, the largest gains occur
when long $f_{\min}$ plateaus delay useful FOCAL admissions; in such
settings, the probabilistic factor may reduce node expansions by about 90\% or
more (e.g., on N-Puzzle and TSP). For the anytime algorithm family, Anytime
Probabilistic Focal Search (APFS) outperforms all tested algorithms in evaluating anytime
methods on GCTSP. We also observe that
the benefit is smaller when the deterministic search already advances
efficiently (e.g., Pancake Sorting), indicating that the probabilistic factor
is most useful when FOCAL admission is a search bottleneck. The PDPS transfer
shows that the mechanism also transfers to potential guidance, although its
common-success effects remain domain- and bound-dependent.
\end{abstract}

\section{Introduction}

Bounded-suboptimal heuristic search trades optimality for a certificate: given a
factor $w\geq 1$, it seeks a solution of cost at most $wC^*$ while using far
less search than A*. Focal Search (FS) realizes this idea with an admissible
primary heuristic and a secondary priority. It maintains
$\mathit{FOCAL}=\{n\in\mathit{OPEN}:f(n)\leq w f_{\min}\}$ and repeatedly
expands the best eligible node under that secondary priority
\citep{pearl1982studies}. Dynamic Potential Search (DPS) instead ranks OPEN by
a potential that changes with the same lower bound $f_{\min}$
\citep{gilon2016dynamic}.

This separation between certification and guidance is powerful, but it creates
a feedback loop. Guided expansions can remain on one minimum-$f$ plateau; while
$f_{\min}$ is unchanged, so is the eligibility threshold $w f_{\min}$. A good
node just outside that threshold cannot benefit from the secondary policy until
the lower-bound frontier advances. The deterministic policy has no explicit
control for making that happen.

We study a minimal intervention to FS: with probability $p$, retain its guided
FOCAL choice; with probability $1-p$, expand a minimum-$f$ OPEN node. This
yields Probabilistic Focal Search (PFS). The intervention does not relax the
solution envelope. Rather, it spends a controlled fraction of selections on
the frontier defining that envelope. The question is empirical as well as
theoretical: forced OPEN-head work may shorten harmful plateaus, but it is
overhead when FS already advances effectively. As a secondary transfer
experiment, we apply the same scheduler to DPS, yielding Probabilistic Dynamic
Potential Search (PDPS). This comparison tests whether the mechanism transfers
to a different guided policy.

We evaluate this idea on N-Puzzle, Pancake Sorting, and the Traveling
Salesperson Problem (TSP), and evaluate its anytime extension on the metric
Generalized Covering TSP (GCTSP). The results show that probabilistic
lower-bound advancement can improve success and substantially reduce search
effort when it admits useful nodes that the guided policy can exploit; APFS
also leads the evaluated anytime methods on GCTSP. We further analyze cases in
which the intervention provides little or no benefit, explaining how an
already sufficient FOCAL or ineffective guidance over newly admitted nodes can
limit the mechanism. PDPS and APDPS serve as secondary transfer experiments
for this dependence on the guided policy.

This paper contributes:
\begin{itemize}
    \item PFS, a tunable Bernoulli scheduler that interleaves guided FOCAL
    selections with minimum-$f$ OPEN selections, together with a
    bounded-suboptimality proof;
    \item a mechanism analysis identifying when the scheduler helps and when
    it adds overhead: minimum-$f$ selections help when they advance the lower
    bound and admit nodes that the guided policy can exploit, but not when
    FOCAL is already sufficient or the new admissions do not improve guided
    choices; and
    \item a cross-domain evaluation supporting both the scheduler's
    effectiveness and the proposed mechanism. On N-Puzzle, PFS raises success
    from 44.3\% to 83.6\%, while reducing penalized mean expansions by 60.7\%
    and capped mean runtime by 69.8\%. Across both TSP datasets, it reaches at
    least 96.4\% success and reduces these metrics by at least 87.2\% and
    89.0\%. Its smaller gains on Pancake, where FOCAL is already sufficient,
    together with the gains and remaining regressions in the DPS family, are
    consistent with the identified boundary conditions. In anytime GCTSP,
    APFS solves 189/234 medium instances versus 82/234 for
    AFS\ifdefined\nostatisticaltests.\else,
    with paired uncertainty estimates.\fi
\end{itemize}
 \section{Background and Related Work}

Each search node $n$ represents a problem state together with its current path
information. Let $g(n)$ be its best known path cost, let $h(n)$ be an
admissible estimate of the remaining cost, and let $f(n)=g(n)+h(n)$. A*
expands a minimum-$f$ node and returns an optimal solution under standard graph
search assumptions \citep{hart1968formal}. Weighted A* (WA*) instead orders by
$q_w(n)=g(n)+w h(n)$ to obtain a $w$-suboptimal solution more quickly
\citep{pohl1970heuristic}.

\paragraph{Focal Search.}
FS maintains the lower bound
$f_{\min}=\min_{n\in\mathit{OPEN}}f(n)$ and the eligible set
\begin{equation}
 \mathit{FOCAL}=\{n\in\mathit{OPEN}:f(n)\leq w f_{\min}\}.
 \label{eq:focal}
\end{equation}
FOCAL can also be defined using an explicit eligibility threshold $C$ as
$\mathit{FOCAL}(C)=\{n\in\mathit{OPEN}:f(n)\leq C\}$. In this paper, we use
the multiplicative definition in Equation~\eqref{eq:focal}.
It selects an eligible node using a secondary priority $d(n)$, which need not
itself be admissible. This decoupling distinguishes focal search from algorithms
whose ordering and certificate use the same heuristic. Anytime Focal Search
reuses this structure while tightening the bound after incumbent solutions
\citep{cohen2018anytime}; our main study concerns the first bounded solution.

\paragraph{Dynamic Potential Search.}
DPS can be viewed as a special case of FS whose secondary priority is the
negative dynamic potential, $d(n)=-u_w(n)$ \citep{gilon2016dynamic}, where
\begin{equation}
 u_w(n)=\frac{w f_{\min}-g(n)}{h(n)}.
 \label{eq:potential}
\end{equation}
For $h(n)>0$, $u_w(n)\geq1$ exactly when
$f(n)\leq w f_{\min}$. Under exact potential ordering, a maximum-potential
node in OPEN is therefore in FOCAL, so DPS can maximize $u_w(n)$ over OPEN
without maintaining FOCAL explicitly. Unlike FS with a fixed secondary
priority, however, $u_w(n)$ depends on $f_{\min}$. Every change in $f_{\min}$
therefore requires the potentials and their ordering to be refreshed before
the next selection.
Following the weighted-graph treatment of DPS
\citep{gilon2017weighted}, we define an $h=0$ node to have potential $+\infty$ only if
$g(n)\leq w f_{\min}$, and $-\infty$ otherwise. This endpoint convention is
essential for a safe goal test.
 \section{Probabilistic Focal Search}

Algorithm~\ref{alg:probabilistic-selection} presents PFS. It retains the OPEN
and FOCAL organization of FS but uses a probability $p$ to choose between two
selection policies. At iteration $t$, it draws
$Z_t\sim\mathrm{Bernoulli}(p)$, where $\Pr(Z_t=1)=p$ and
$\Pr(Z_t=0)=1-p$. When $Z_t=1$, PFS follows the FS policy and selects the
minimum-$d$ node in FOCAL. When $Z_t=0$, it selects a minimum-$f$ node from
OPEN. Thus $p=1$ recovers FS, while $p=0$ uses A*'s primary ordering.

\begin{algorithm}[tb]
\caption{Probabilistic Focal Search (PFS)}
\label{alg:probabilistic-selection}
\begin{algorithmic}[1]
\REQUIRE start node $s_0$, $w\geq1$, and $p\in[0,1]$
\STATE $g(s_0)\leftarrow0$; $OPEN\leftarrow\{s_0\}$; $CLOSED\leftarrow\emptyset$
\STATE $f_{\min}\leftarrow f(s_0)$; $FOCAL\leftarrow\{s_0\}$
\WHILE{$OPEN\neq\emptyset$}
    \STATE draw $Z\sim\mathrm{Bernoulli}(p)$
    \IF{$Z=1$}
        \STATE $n\leftarrow\arg\min_{u\in FOCAL}d(u)$
    \ELSE
        \STATE $n\leftarrow\arg\min_{u\in OPEN}f(u)$
    \ENDIF
    \IF{$Goal(n)$}
        \STATE \textbf{return} $n$
    \ENDIF
    \STATE remove $n$ from OPEN and FOCAL; add $n$ to CLOSED
    \STATE relax successors; insert/reopen improved successors in OPEN and in FOCAL
    when currently eligible
    \IF{$OPEN\neq\emptyset$}
        \STATE $f'_{\min}\leftarrow\min_{u\in OPEN}f(u)$
        \IF{$f'_{\min}>f_{\min}$}
            \STATE add newly eligible nodes to FOCAL
        \ELSIF{$f'_{\min}<f_{\min}$}
            \STATE rebuild FOCAL using threshold $w f'_{\min}$
        \ENDIF
        \STATE $f_{\min}\leftarrow f'_{\min}$
    \ENDIF
\ENDWHILE
\STATE \textbf{return} failure
\end{algorithmic}
\end{algorithm}

The initialization places the start node in OPEN and FOCAL and sets
$f_{\min}=f(s_0)$. At a non-goal expansion, the selected node is removed from
both frontier structures and its successors are generated. Standard duplicate
relaxation and reopening are abbreviated: every inserted or improved successor
is placed in FOCAL exactly when it satisfies the current threshold. The
algorithm then recomputes $f_{\min}$. If it increases, the enlarged threshold
admits the newly eligible OPEN nodes; if it decreases after reopening, FOCAL is
rebuilt to remove nodes outside the reduced threshold. Both selection branches
choose an eligible node, since the OPEN head has $f=f_{\min}$, so a selected
goal satisfies the same $w f_{\min}$ envelope as in FS.

\paragraph{Secondary transfer to DPS.}
To test whether the scheduler transfers beyond FS, we also instantiate it in
DPS as PDPS. PDPS uses the same Bernoulli choice and minimum-$f$ OPEN branch. It
changes only the guided branch: when $Z=1$, it selects
$n=\arg\max_{u\in FOCAL}u_w(u)$ under Equation~\ref{eq:potential}. Under exact
potential ordering, the maximum over OPEN lies in FOCAL, so this restriction
preserves the DPS choice while keeping only eligible potential groups active.
Thus $p=1$ recovers DPS and $p=0$ again uses A*'s primary ordering. Regardless
of queue ordering, PDPS accepts a selected goal only after the exact check
$g(n)\leq w f_{\min}$; otherwise the goal remains in OPEN.

\begin{assumption}
Edge costs are nonnegative, and the primary heuristic $h$ is admissible and
nonnegative. Standard duplicate detection and reopening maintain the frontier
invariant $f_{\min}\leq C^*$ until a goal is accepted, where $C^*$ is the
optimal solution cost.
\end{assumption}

\begin{theorem}[Certificate preservation]
Any goal returned by first-solution PFS has cost at most $wC^*$.
\end{theorem}

\begin{proof}
Let $n_g$ be the returned goal. The acceptance condition and frontier invariant
give
\begin{equation}
 C=g(n_g)=f(n_g)\leq w f_{\min}\leq wC^*.
\end{equation}
The randomized branch changes which node is considered, but neither the lower
bound nor the acceptance condition. The result therefore holds for every
$p\in[0,1]$.
\end{proof}

The same certificate transfers to PDPS because it uses the same lower bound and
exact goal-acceptance condition. Appendix~\ref{app:dps-eligibility} verifies
that its guided DPS choice is also eligible under exact potential ordering.

\begin{lemma}
PFS and FS cannot return a goal while $f_{\min}<C^*/w$.
\end{lemma}

\begin{proof}
If $f_{\min}<C^*/w$, every acceptable goal would have
$g(n)\leq w f_{\min}<C^*$. This contradicts the definition of $C^*$.
\end{proof}

This necessary threshold motivates the OPEN-head branch of PFS. By expanding a
current minimum-$f$ node with probability $1-p$, PFS attempts to clear the
minimum-$f$ plateau and raise $f_{\min}$ toward $C^*/w$, thereby raising the
threshold and potentially admitting additional guided nodes. Once
$f_{\min}\geq C^*/w$, any optimal goal node that
has already been generated and remains in OPEN satisfies
$f(n)=C^*\leq w f_{\min}$ and is therefore in FOCAL.

PDPS inherits the same threshold motivation; the transfer experiments test
whether its potential-based guided policy can exploit newly eligible nodes as
effectively as PFS.
 \section{Experiments}

We evaluate PFS and PDPS through four questions: whether they improve their
deterministic parents, FS and DPS (RQ1); how lower-bound advancement changes
FOCAL and explains their behavior (RQ2); how performance varies with $p$
(RQ3); and whether the scheduler benefits their anytime variants, APFS and
APDPS (RQ4). PFS is the primary method, while PDPS tests whether the mechanism
transfers to potential-based guidance.

\begin{table*}[t]
\centering
\small
\setlength{\tabcolsep}{5pt}
\ifdefined\nostatisticaltests
    \caption{First-solution performance over the seven tested bounds at
    $p=0.70$. P/D denotes
    probabilistic/deterministic; K and M denote thousands and millions.
    Mean expansions use all runs. Within each setting, an unsuccessful run is
    assigned the largest expansion count observed across the two compared
    algorithms, and Exp. ratio is the ratio of the resulting P/D means.
    Unsuccessful runs contribute the time limit to Time P/D.}
\else
    \caption{Summary for $p=0.70$ and
    $w\in\{1.05,1.10,1.25,1.50,2.00\}$. Ratios are
    probabilistic/deterministic; values below 1 are better. Success P/D uses all
    matched runs, while expansion statistics use instances solved by both
    methods. Coverage W/T/L counts settings with higher/equal/lower
    probabilistic success, and median $\Delta$ pp is probabilistic minus
    deterministic. Median exp. P/D is the median across settings of each
    method's within-setting median (K denotes thousands); exp. ratio is the
    median of setting-level paired ratios. Capped mean time assigns the time
    limit to unsuccessful runs. The rows aggregate 15 N-Puzzle, 20 Pancake, ten
    PFS/FS TSP settings; five PDPS/DPS settings form the secondary transfer
    comparison.}
\fi
\label{tab:headline}
\ifdefined\nostatisticaltests
\begin{tabular}{llrrrr}
\hline
Pair & Domain & Success P/D (\%) & Mean exp. P/D & Exp. ratio & Time P/D (s) \\
\hline
PFS/FS & N-Puzzle & 83.6/44.3 & 748.0K/1.902M & 0.393 & 51.11/169.25 \\
       & Pancake  & 72.0/70.2 & 21.4K/24.5K & 0.874 & 84.51/90.04 \\
       & TSP-40   & 99.4/78.6 & 25.8K/446.2K & 0.058 & 3.56/82.98 \\
       & TSP-50   & 96.4/58.3 & 67.2K/526.8K & 0.128 & 14.93/135.36 \\
PDPS/DPS & N-Puzzle & 77.7/65.4 & 910.2K/1.203M & 0.757 & 71.41/112.10 \\
         & Pancake  & 95.8/94.0 & 4.1K/5.1K & 0.793 & 13.11/23.65 \\
         & TSP-40   & 96.0/89.3 & 9.1K/15.5K & 0.586 & 22.21/42.87 \\
         & TSP-50   & 88.7/79.1 & 20.2K/30.7K & 0.658 & 44.81/70.51 \\
\hline
\end{tabular}
\else
\begin{tabular}{llrrrrrr}
\hline
Pair & Domain & Success P/D (\%) & Coverage W/T/L & Median $\Delta$ pp & Median exp. P/D & Exp. ratio & Time P/D (s) \\
\hline
PFS/FS & N-Puzzle & 82.6/46.8 & 15/0/0 & +36.0 & 54.4K/600.7K & 0.135 & 53.98/161.71 \\
       & Pancake  & 72.7/70.0 & 4/13/3 &  +0.0 & 201/148 & 1.346 & 82.43/90.48 \\
       & TSP      & 99.2/81.9 & 9/1/0  & +17.0 & 1.8K/110.3K & 0.034 & 4.76/73.55 \\
PDPS/DPS & N-Puzzle & 72.3/64.7 & 12/2/1 & +6.0 & 192.1K/473.6K & 0.607 & 87.45/111.13 \\
         & Pancake  & 94.0/93.2 & 2/17/1 & +0.0 & 188/144 & 1.322 & 18.53/20.72 \\
         & TSP      & 99.0/98.8 & 1/4/0  & +0.0 & 293/233 & 1.309 & 5.89/8.23 \\
\hline
\end{tabular}
\fi
\end{table*}

\begin{figure*}[t]
    \centering
    \includegraphics[width=0.94\textwidth]{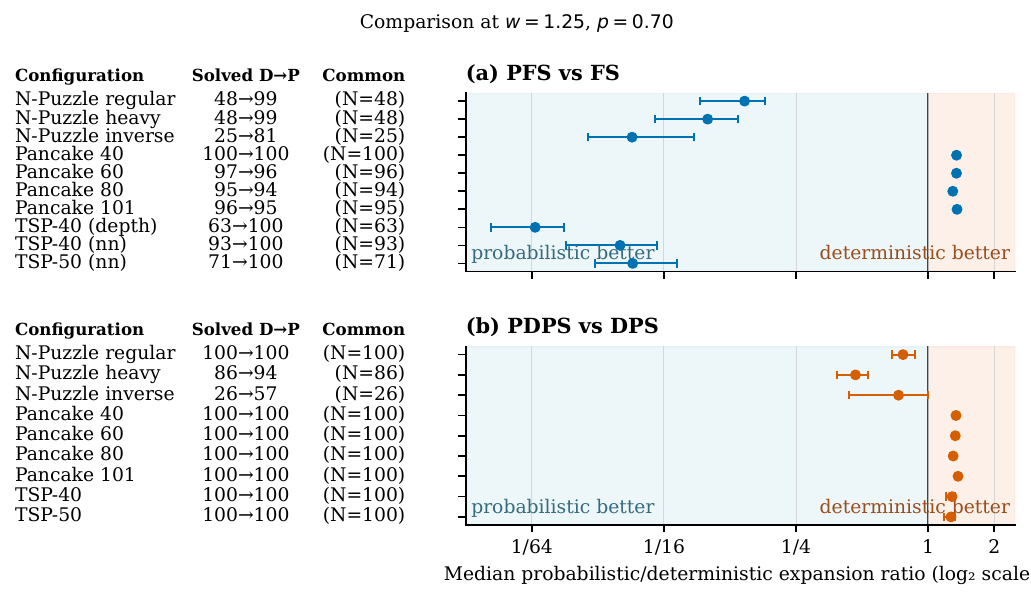}
    \caption{Probabilistic/deterministic node-expansion ratios at $w=1.25$,
    $p=0.70$.
    \ifdefined\nostatisticaltests
    Labels report deterministic$\rightarrow$probabilistic solved counts.
    \else
    Whiskers are 95\% paired-bootstrap intervals; labels report
    deterministic$\rightarrow$probabilistic solved counts and the number $N$
    solved by both.
    \fi
    The DPS-family rows use the Envelope FOCAL implementation. The vertical
    line denotes equal effort: points to its left favor the
    probabilistic method, whereas points to its right favor the deterministic
    counterpart.}
    \label{fig:central-effects}
\end{figure*}

\subsection{Experimental Setup}

\begin{figure*}[p]
    \centering
    \begin{minipage}{0.32\textwidth}
        \centering
        \includegraphics[width=\linewidth]{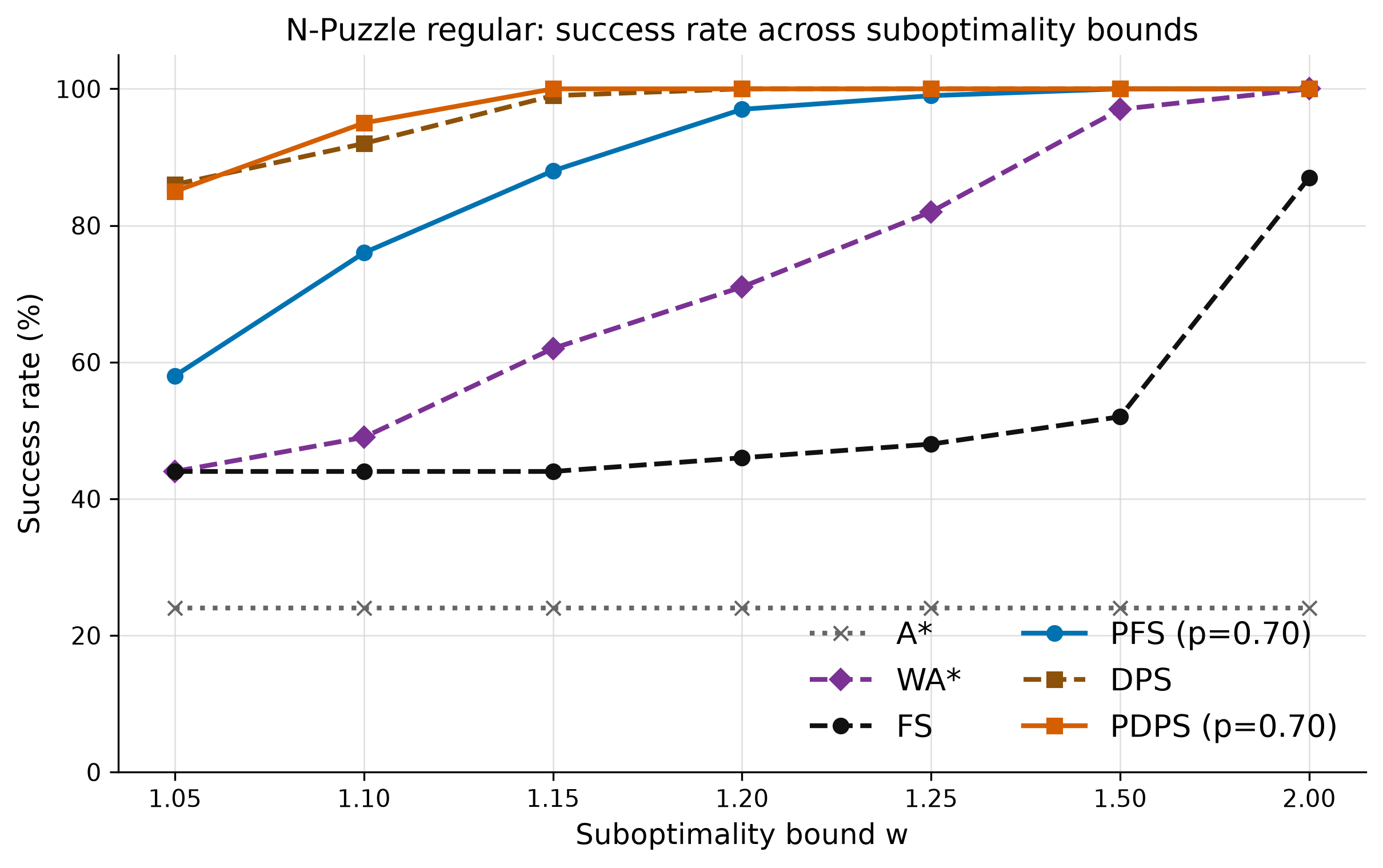}\\[-0.5ex]
        \small (a) Regular costs
    \end{minipage}\hfill
    \begin{minipage}{0.32\textwidth}
        \centering
        \includegraphics[width=\linewidth]{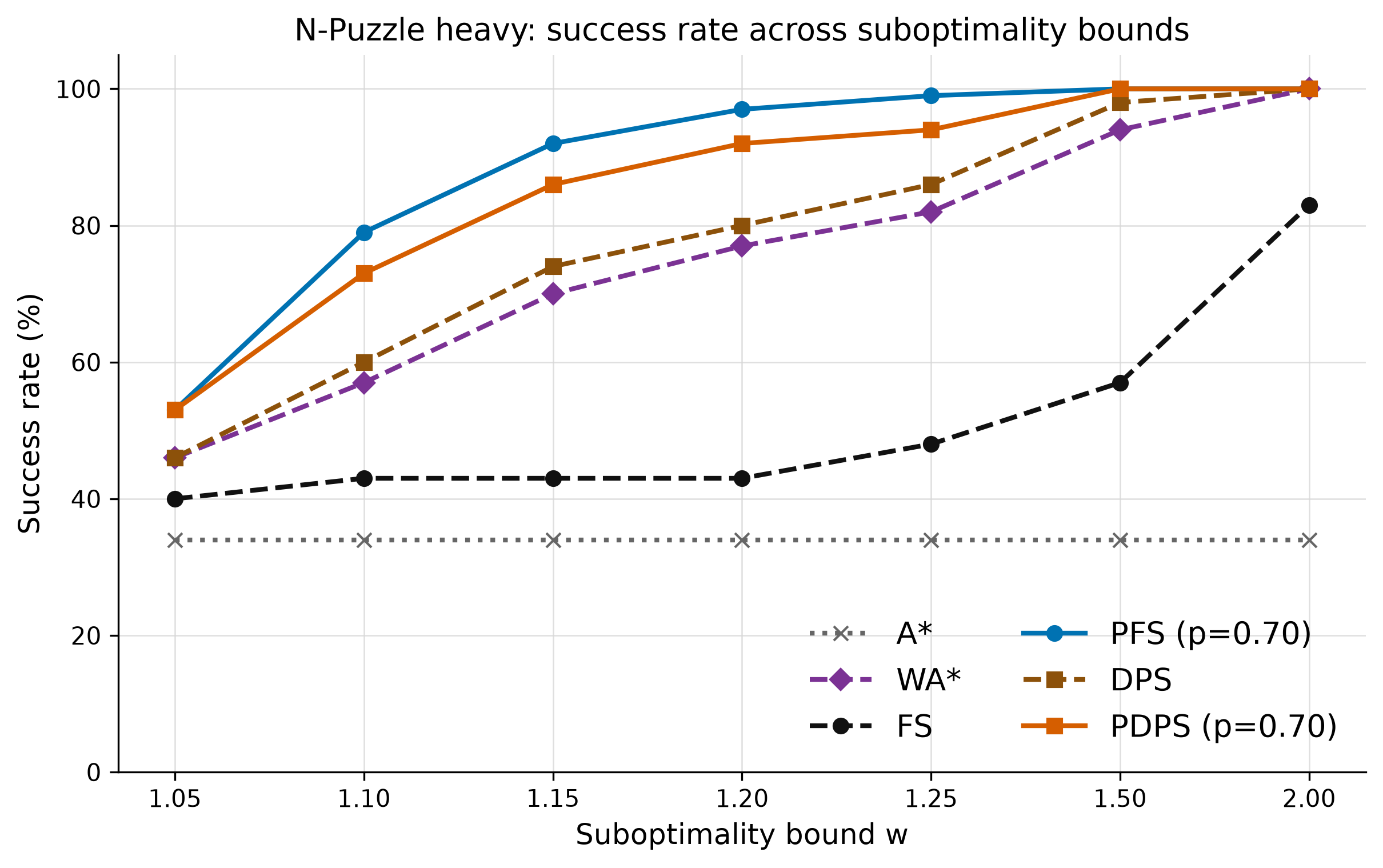}\\[-0.5ex]
        \small (b) Heavy costs
    \end{minipage}\hfill
    \begin{minipage}{0.32\textwidth}
        \centering
        \includegraphics[width=\linewidth]{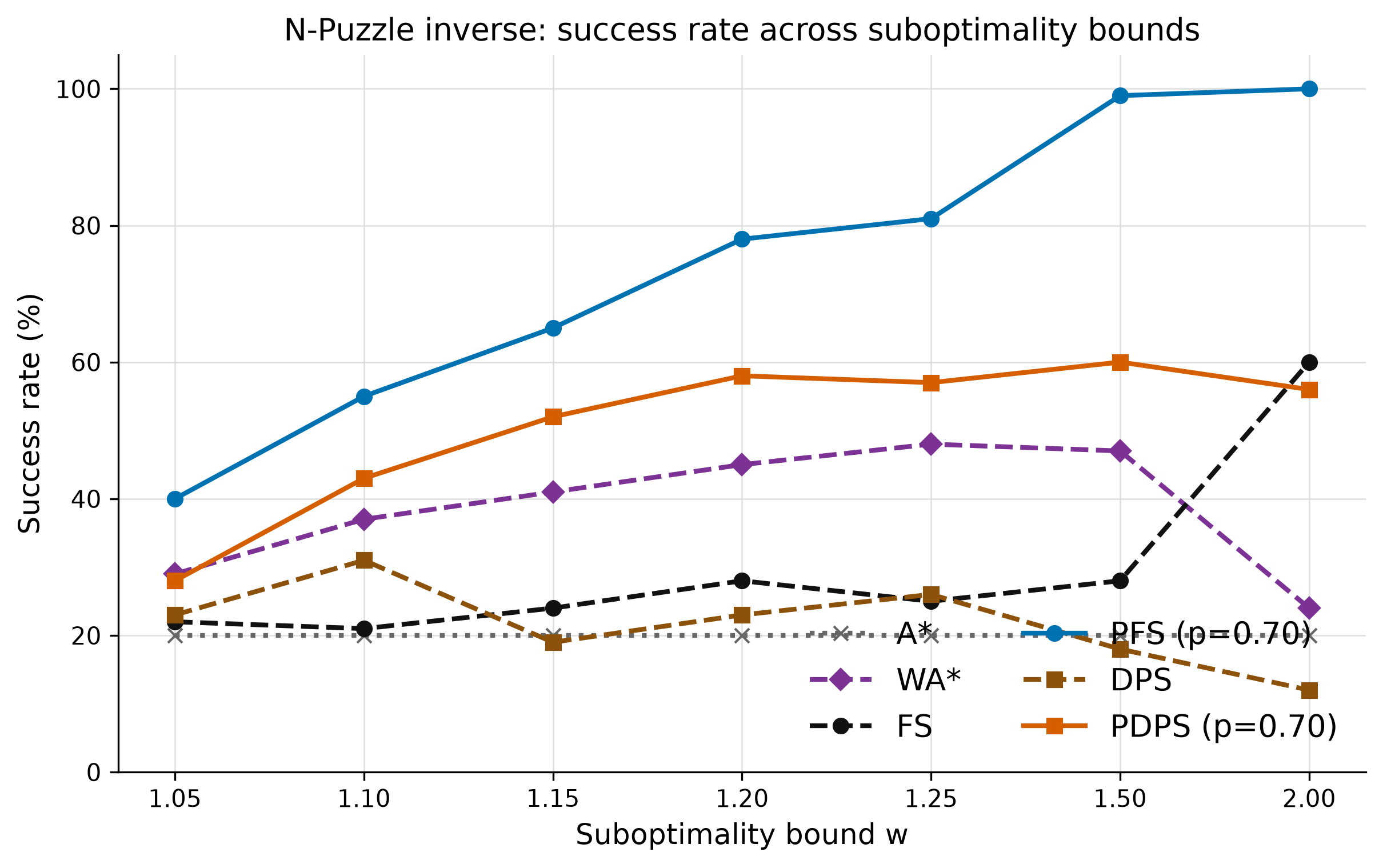}\\[-0.5ex]
        \small (c) Inverse costs
    \end{minipage}
    \caption{Success rate on Korf's 100 N-Puzzle instances across
    suboptimality bounds. PFS and PDPS use $p=0.70$; the DPS-family curves use
    Envelope FOCAL. Deterministic baselines use the same weight grid where
    applicable.}
    \label{fig:npuzzle-success-p070}

    \begin{minipage}{0.32\textwidth}
        \centering
        \includegraphics[width=\linewidth]{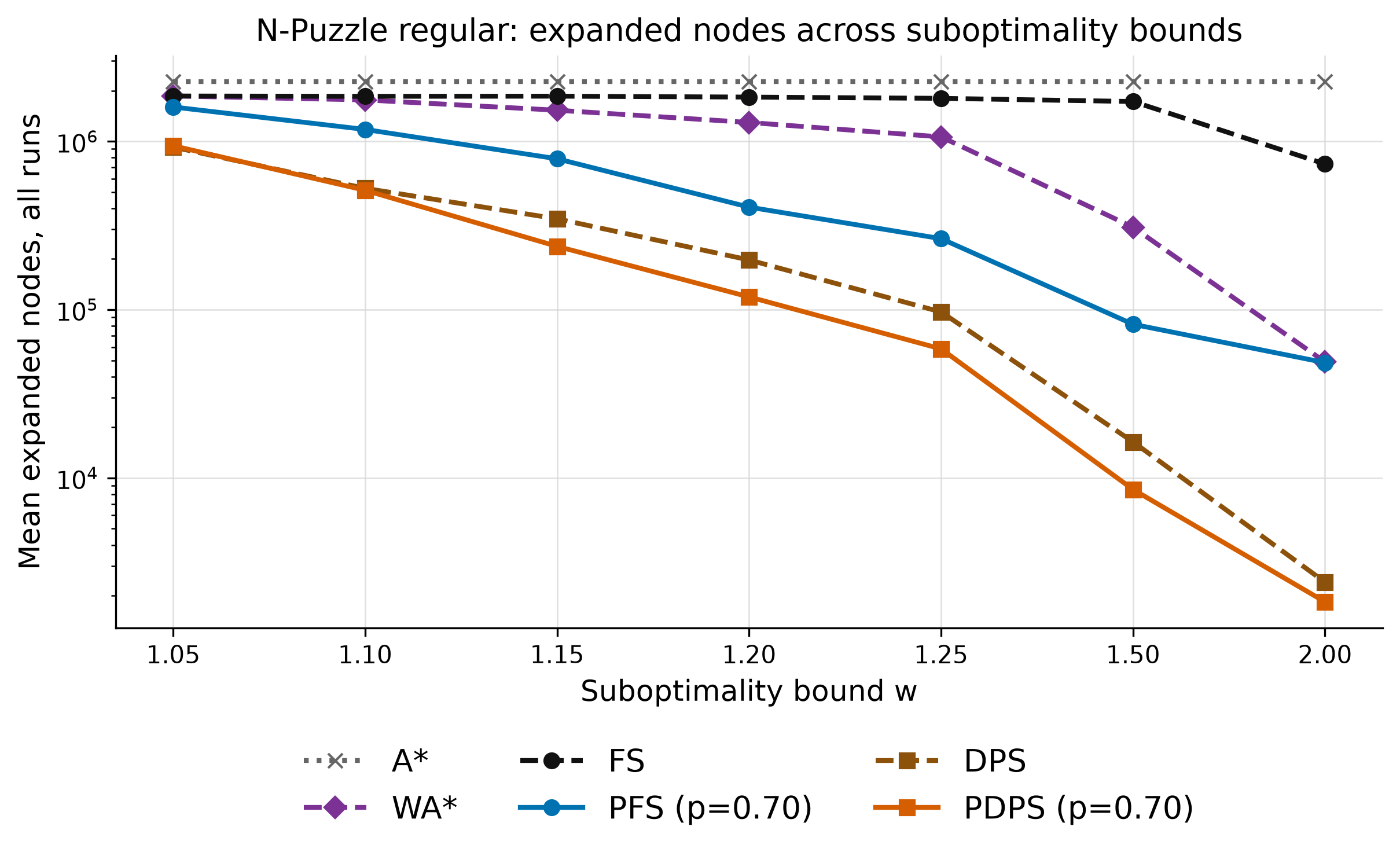}\\[-0.5ex]
        \small (a) Regular costs
    \end{minipage}\hfill
    \begin{minipage}{0.32\textwidth}
        \centering
        \includegraphics[width=\linewidth]{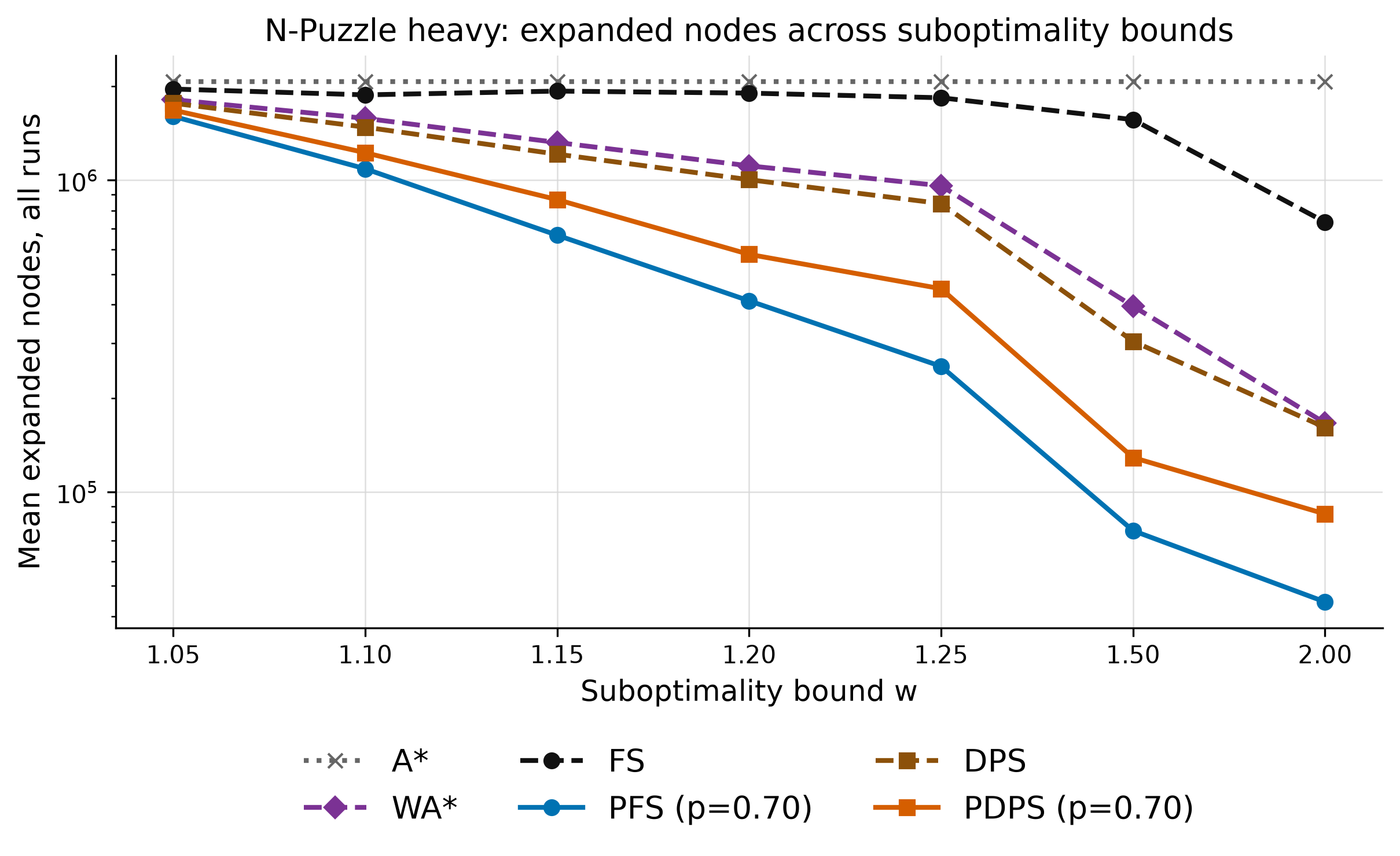}\\[-0.5ex]
        \small (b) Heavy costs
    \end{minipage}\hfill
    \begin{minipage}{0.32\textwidth}
        \centering
        \includegraphics[width=\linewidth]{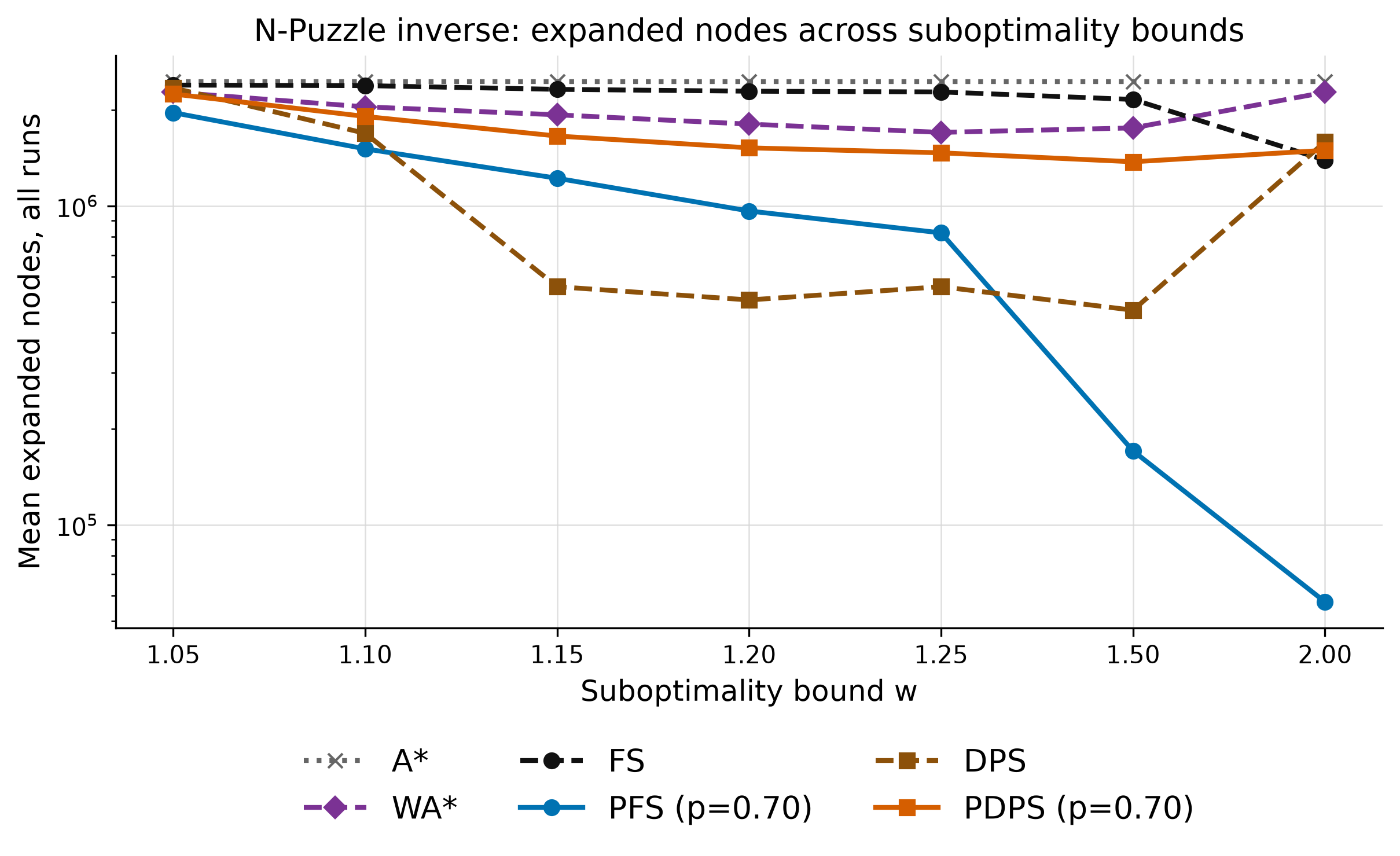}\\[-0.5ex]
        \small (c) Inverse costs
    \end{minipage}
    \caption{Mean node expansions over all Korf N-Puzzle runs, including
    unsuccessful executions.}
    \label{fig:npuzzle-expanded-all-p070}

    \begin{minipage}{0.32\textwidth}
        \centering
        \includegraphics[width=\linewidth]{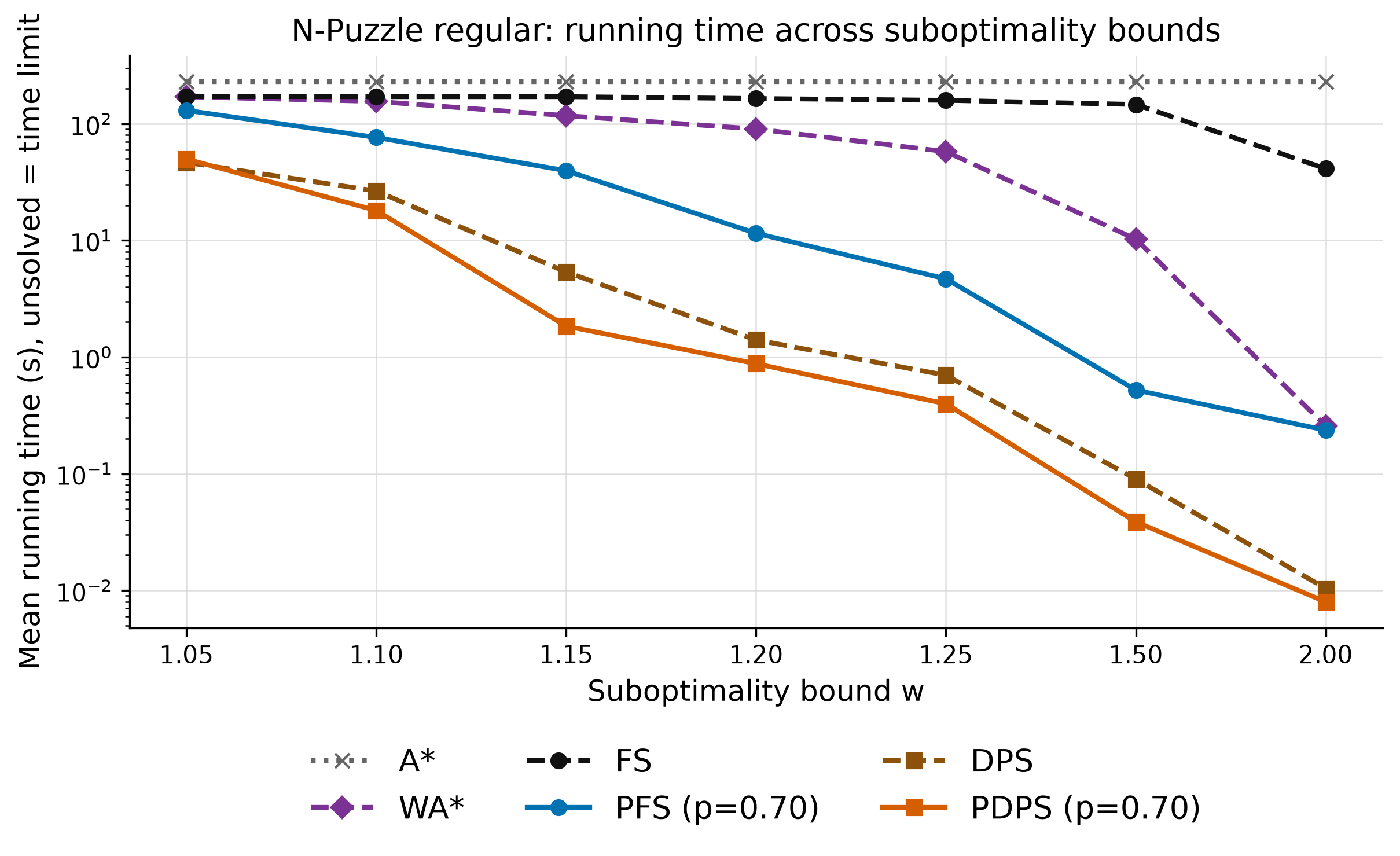}\\[-0.5ex]
        \small (a) Regular costs
    \end{minipage}\hfill
    \begin{minipage}{0.32\textwidth}
        \centering
        \includegraphics[width=\linewidth]{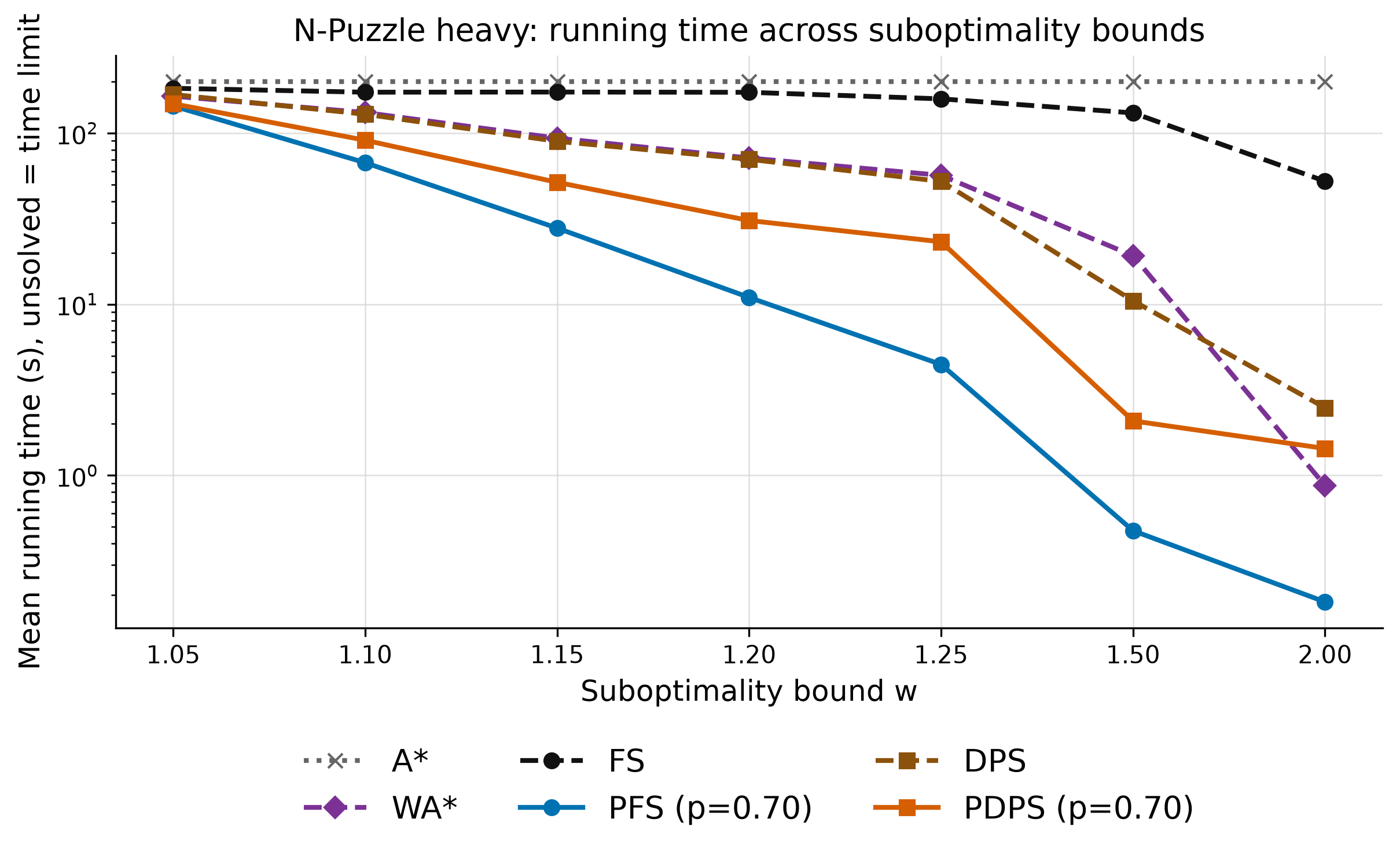}\\[-0.5ex]
        \small (b) Heavy costs
    \end{minipage}\hfill
    \begin{minipage}{0.32\textwidth}
        \centering
        \includegraphics[width=\linewidth]{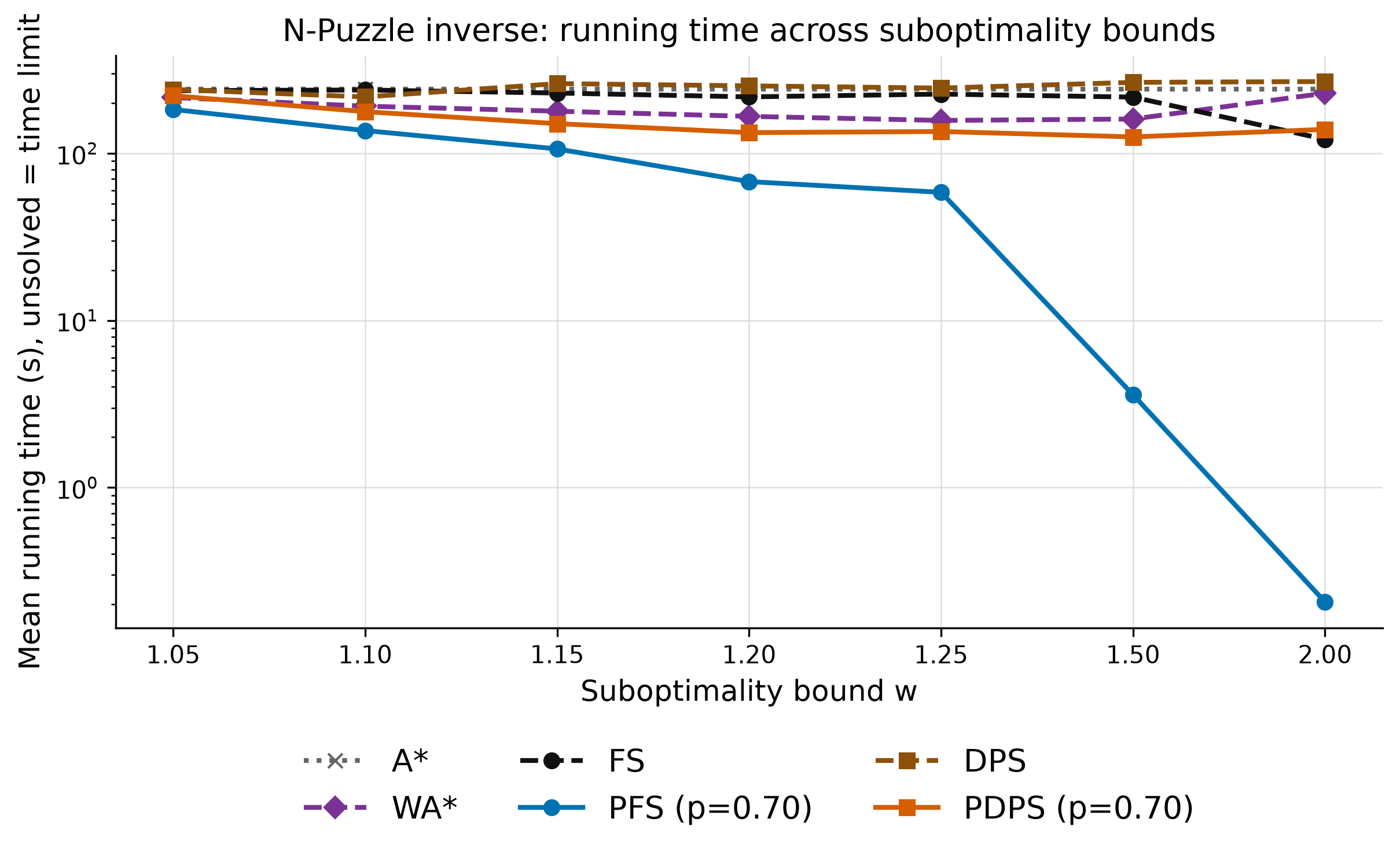}\\[-0.5ex]
        \small (c) Inverse costs
    \end{minipage}
    \caption{Mean running time over all Korf N-Puzzle runs, with unsuccessful
    executions assigned the time limit.}
    \label{fig:npuzzle-runtime-p070}

    \begin{minipage}[t]{0.48\textwidth}
        \centering
        \includegraphics[width=0.94\linewidth]{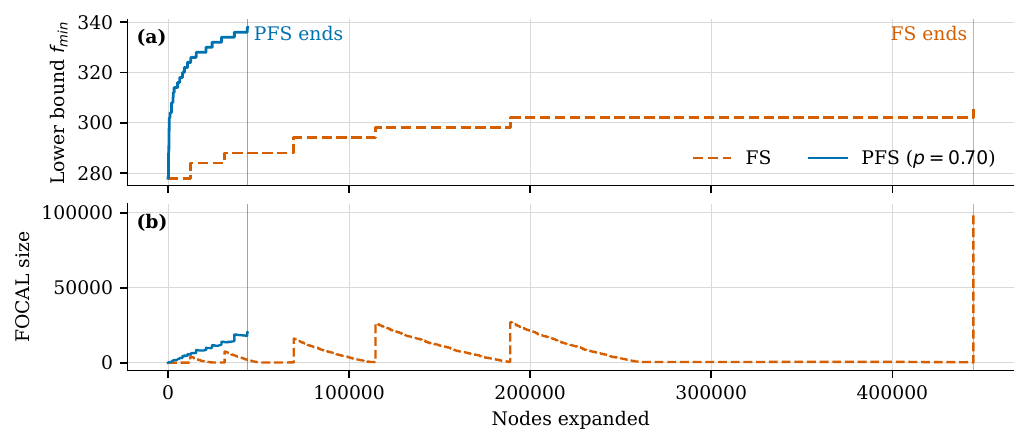}
        \captionsetup{width=\linewidth}
        \caption{Lower-bound and FOCAL-size trajectories for heavy N-Puzzle
        \texttt{npuzzle\_0039}, $w=1.25$, $p=0.70$, seed 42.
        \ifdefined\nostatisticaltests\else The instance is rank 24/48 by paired
        PFS/FS expansion ratio.\fi\ PFS terminates at 43,972 expansions and FS
        at 444,774.}
        \label{fig:mechanism-trajectory}
    \end{minipage}\hfill
    \begin{minipage}[t]{0.48\textwidth}
        \centering
        \includegraphics[width=0.72\linewidth]{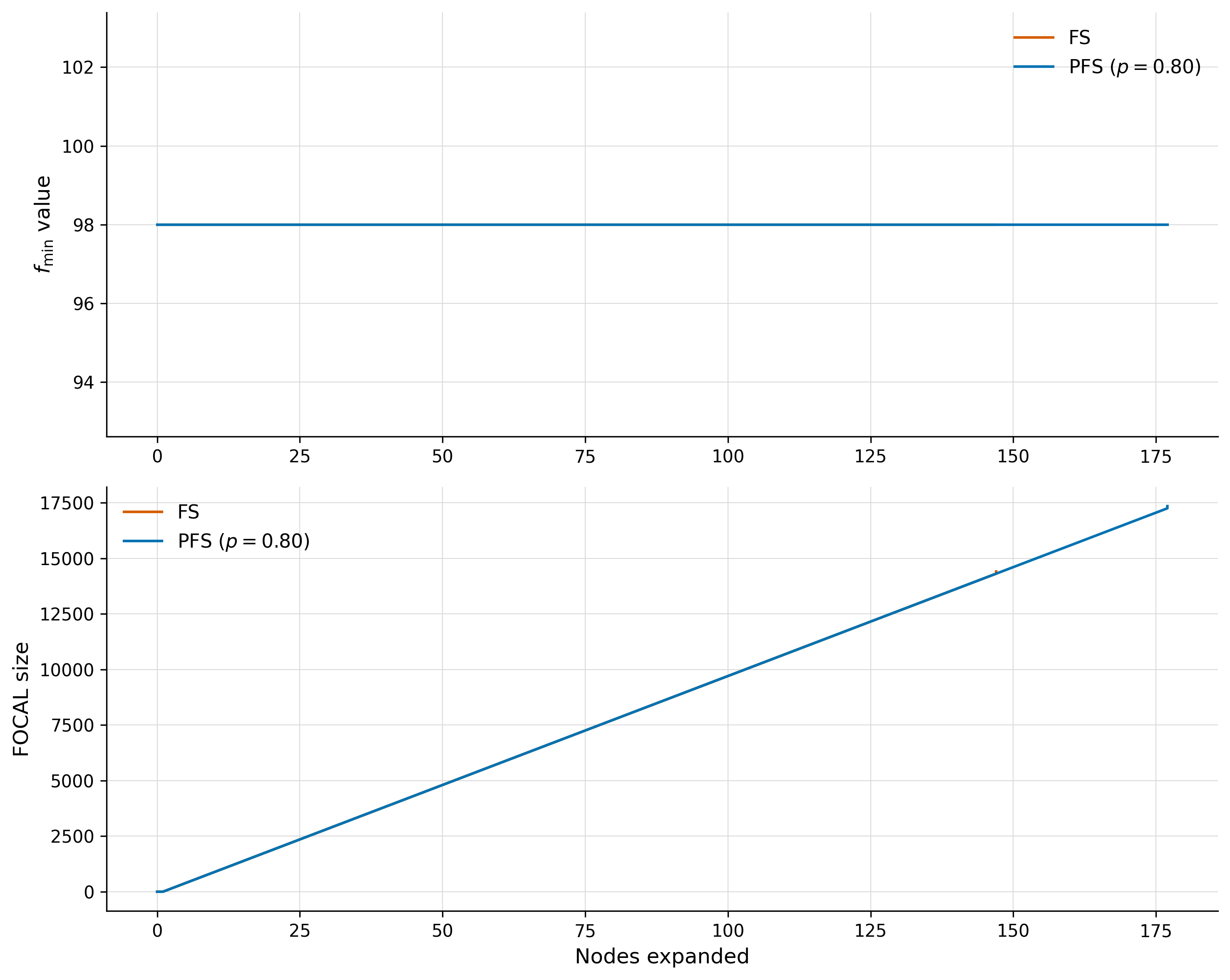}
        \captionsetup{width=\linewidth}
        \caption{Lower-bound and FOCAL trajectories for Pancake-101
        \texttt{pancake\_0071} ($w=1.25$, $p=0.80$). FS expands 147
        nodes and PFS expands 177.}
        \label{fig:pancake-trajectory}
    \end{minipage}

\end{figure*}

The benchmarks cover Korf's 100 15-Puzzle instances
\citep{korf1985depth} with unit, tile-value, and inverse-tile costs
\citep{gilon2017weighted}; 100-instance Pancake sets of sizes 40, 60, 80, and
101; and two Euclidean TSP datasets, each has 100 instances, containing
40 and 50 cities, respectively. The anytime study uses 81 small and 234 medium metric instances
of the Generalized Covering Traveling Salesperson Problem (GCTSP), generated
from TSPLIB instances \citep{reinelt1991tsplib}. In GCTSP, a tour visits facilities until they
collectively cover a required customer quota and then returns to the depot
\citep{shaelaie2014generalized,cohen2018anytime}.

For both heuristics, N-Puzzle uses cost-adjusted Manhattan distance plus linear
conflict \citep{hansson1992criticizing}, while Pancake uses GAP
\citep{helmert2010landmark}. For TSP,
the primary heuristic combines a minimum spanning tree (MST) over the unvisited
cities with connections from the current city and back to the depot. The MST
therefore lower-bounds the cost required to connect the cities that remain in
an unfinished tour. FS and PFS use depth or nearest-neighbor secondary guidance
on TSP-40 and nearest-neighbor guidance on TSP-50. DPS and PDPS use the same
primary heuristic as their FS and PFS counterparts; their node ordering
does not use the secondary heuristic. DPS and PDPS use the same FOCAL
implementation, called Envelope FOCAL. It orders OPEN by $f$ and selects
the maximum-potential group among FOCAL-eligible states.

We test $w\in\{1.05,1.10,1.15,1.20,1.25,1.50,2.00\}$ and
$p\in\{0.60,0.70,0.80\}$. We use $p=0.70$ as the main setting for all domains
and report $p=0.60$ and $p=0.80$ in the sensitivity analysis;
\ifdefined\nostatisticaltests the seven-bound first-solution results contain
50,400 runs.\else the canonical five-bound first-solution cohort contains
34,000 runs.\fi\ Generated-node limits are five million for N-Puzzle and Pancake, and
twenty million for TSP and anytime GCTSP. The datasets and source code are
provided in the supplementary material.

\paragraph{Computing environment.}
Algorithms are implemented in C++17 (\texttt{g++ -O3}) and evaluated under Linux on Intel i5-13500 and AMD Ryzen Threadripper PRO 5975WX workstations. Each run uses a single dedicated CPU thread with a 300-second wall-clock limit and memory monitoring.

\ifdefined\nostatisticaltests
\subsection{RQ1: Does PFS Improve Success and Search Effort?}
\else
\subsection{RQ1: Does PFS Improve Coverage and Search Effort?}
\fi

\ifdefined\nostatisticaltests
Table~\ref{tab:headline} reports success, expansions, and capped running
time over all seven tested bounds. Each row of
Figure~\ref{fig:central-effects}
places the probabilistic/deterministic expansion ratio beside the corresponding
solved counts, so efficiency can be read together with the ability to find a
solution.
PFS gives clear improvements on N-Puzzle and TSP. On N-Puzzle, it raises
success from 44.3\% to 83.6\%, reduces mean penalized expansions from 1.902M
to 748.0K (ratio 0.393), and reduces capped mean time from 169.25 to 51.11
seconds. On TSP-40, success rises from 78.6\% to 99.4\%, the expansion ratio is
0.058, and mean time falls from 82.98 to 3.56 seconds. TSP-50 shows the same
pattern: 96.4\% versus 58.3\% success, a 0.128 expansion ratio, and 14.93
versus 135.36 seconds. On Pancake, the two methods have similar success
(72.0\% versus 70.2\%). Under Table~\ref{tab:headline}'s all-run penalty, PFS
has a lower expansion ratio (0.874) and capped mean time (84.51 versus 90.04
seconds). RQ2 explains why this aggregate result does not imply lower effort
on the Pancake instances solved by both methods.

The all-run comparison also favors PDPS over DPS.
PDPS/DPS success is 77.7\%/65.4\% on N-Puzzle, 95.8\%/94.0\% on Pancake,
96.0\%/89.3\% on TSP-40, and 88.7\%/79.1\% on TSP-50. The corresponding
penalized expansion ratios are 0.757, 0.793, 0.586, and 0.658, while capped
mean time is also lower in every domain. Runtime for both methods is sensitive
to the FOCAL implementation because each change in $f_{\min}$ changes the
potential ordering.

Using TSP-50 as an example, Table~\ref{tab:tsp50-detail} illustrates both the
benefit of PFS and when probabilistic DPS helps or adds work. For $w\leq1.25$,
PFS solves 77--100 instances versus
25--71 for FS, with mean paired expansion ratios of 0.061--0.285 and runtime
ratios of 0.039--0.332 on commonly solved instances. At $w=2$, both methods solve every
instance, but the expansion and runtime ratios rise above 1. PDPS similarly
improves DPS at the tight bounds: for $w\leq1.15$, it solves 40--99 instances
versus 12--95, and both paired ratios remain below 1. Once both solve all
instances ($w\geq1.20$), however, the OPEN-head branch adds work and the paired
ratios exceed 1. Its greater tight-bound success nevertheless produces the
favorable all-run TSP-50 values in Table~\ref{tab:headline}.

Figure~\ref{fig:central-effects} fixes $w=1.25$ to show how the expansion
effect varies by configuration rather than suggesting that the probabilistic
scheduler is uniformly beneficial. PFS lies well to the left of equal effort
on N-Puzzle and TSP but
slightly to the right on Pancake. At $w=1.25$, PDPS lies to the left on
N-Puzzle but to the right on Pancake and TSP, where both DPS variants solve
every instance. Advancing $f_{\min}$ helps only when the guided policy can
exploit newly eligible states; otherwise, OPEN-head selections add work.
Figures~\mbox{\ref{fig:npuzzle-success-p070}--
\ref{fig:npuzzle-runtime-p070}} conclude RQ1 by comparing N-Puzzle success,
node expansions, and runtime across algorithms and bounds.

\begin{table}[t]
\centering
\small
\setlength{\tabcolsep}{2.5pt}
\renewcommand{\arraystretch}{1.08}
\captionsetup{skip=3pt}
\caption{TSP-50 results. Expansion and runtime entries are mean paired
probabilistic/deterministic ratios on commonly solved instances; values below
1 favor the probabilistic method.}
\label{tab:tsp50-detail}
\begin{tabular*}{\columnwidth}{@{\extracolsep{\fill}}r rrr @{\hspace{6pt}} rrr@{}}
\hline
& \multicolumn{3}{c}{PFS/FS} & \multicolumn{3}{c}{PDPS/DPS} \\
\cline{2-4}\cline{5-7}
$w$ & Solved & Exp. & Time & Solved & Exp. & Time \\
\hline
1.05 & 77/25   & 0.061 & 0.039 & 40/12   & 0.757 & 0.390 \\
1.10 & 98/24   & 0.112 & 0.129 & 82/47   & 0.801 & 0.174 \\
1.15 & 100/38  & 0.154 & 0.213 & 99/95   & 0.928 & 0.923 \\
1.20 & 100/55  & 0.285 & 0.307 & 100/100 & 1.437 & 1.146 \\
1.25 & 100/71  & 0.165 & 0.332 & 100/100 & 1.635 & 1.453 \\
1.50 & 100/95  & 0.878 & 1.475 & 100/100 & 2.032 & 2.051 \\
2.00 & 100/100 & 1.262 & 1.923 & 100/100 & 1.298 & 1.146 \\
\hline
\end{tabular*}
\end{table}

\else
Table~\ref{tab:headline} summarizes cross-domain coverage, paired expansions,
and capped running time over five reported bounds. Figure~\ref{fig:central-effects}
shows configuration-level expansion effects at $w=1.25$, while
Figures~\ref{fig:npuzzle-success-p070}--\ref{fig:npuzzle-runtime-p070} show the
N-Puzzle results over all seven tested bounds.

PFS gives the clearest gains. It raises success from 46.8\% to 82.6\% on
N-Puzzle and from 81.9\% to 99.2\% on TSP; median paired expansion ratios are
0.135 and 0.034. With unsuccessful runs assigned the time limit, its mean times
are 53.98 versus 161.71 seconds and 4.76 versus 73.55 seconds, respectively.

\paragraph{Secondary DPS transfer.}
PDPS gives a smaller N-Puzzle improvement: success rises from 64.7\% to 72.3\%
and the expansion ratio is 0.607. On Pancake, PFS and PDPS have similar success
to their counterparts but expansion ratios of 1.346 and 1.322. PDPS/TSP likewise
has near-equal success and a 1.309 ratio. The deterministic methods are therefore
more expansion-efficient in these comparisons. The absolute workloads are small:
on Pancake, the setting-level median counts are 148 for FS, 201 for PFS, 144 for
DPS, and 188 for PDPS; on 40-city TSP, DPS and PDPS likewise require only 233
and 293 expansions. Thus the ratios above 1 represent small absolute expansion
differences in these cases.

Heuristic discrimination is consistent with these domain differences. Manhattan-based
secondary keys directly estimate remaining tile movement. Pancake's
integer-valued GAP key is more tie-prone: a prefix reversal changes only its
endpoint boundary, so a state's successors take at most three GAP values.
DPS/PDPS combine GAP with $g$ in their potential and can distinguish equal-$h$
nodes at different depths, consistent with the higher Pancake coverage of the
DPS family.

At $w=1.25$ (Figure~\ref{fig:central-effects}), PFS solves 279/300 N-Puzzle
instances versus 121/300 for FS and 200/200 TSP runs versus 156/200. Its heavy
N-Puzzle expansion ratio is 0.099 (95\% CI $[0.057,0.136]$
\ifdefined\nostatisticaltests\else, $q=3.2\times10^{-12}$\fi); regular and
inverse ratios are 0.146 and 0.045. PDPS
solves 249/300 versus 228/300, with ratios 0.607, 0.473, and 0.526.

Thus PFS provides the main improvement in aggregate success and capped mean
time, although its expansion gains remain domain-dependent. The secondary PDPS
results are mixed and delimit how far the scheduler transfers across guided
policies. Returned costs may also be higher while remaining inside the requested
bound; the main benefit is faster discovery of a bounded solution.
\fi

\subsection{RQ2: How Does Lower-Bound Advancement Change FOCAL?}

The OPEN-head branch helps only when clearing a minimum-$f$ plateau raises
$f_{\min}$, enlarges the threshold $w f_{\min}$, and admits nodes that the
guided FOCAL policy can exploit. Otherwise, an OPEN-head selection merely
replaces a guided expansion. Figures~\ref{fig:mechanism-trajectory} and
\ref{fig:pancake-trajectory} illustrate these two cases.

On heavy N-Puzzle (Figure~\ref{fig:mechanism-trajectory}), PFS clears
$f_{\min}$ plateaus earlier. Each threshold increase can admit a batch of
nodes into FOCAL, after which the secondary heuristic guides the search through
the enlarged eligible set. PFS terminates after 43,972 expansions, compared
with 444,774 for FS. This is the intended mechanism: lower-bound advancement
gives the guided policy earlier access to useful nodes.

On Pancake (Figure~\ref{fig:pancake-trajectory}), integer path costs and GAP
values place many states at the same $f$. Here $f_{\min}=98$ throughout, and
the initial threshold $w f_{\min}=122.5$ already exceeds the returned solution
cost 117. Although the goal is generated later, it is immediately eligible
without any increase in $f_{\min}$. PFS therefore performs the same 147 guided
expansions as FS plus 30 OPEN-head expansions, without creating a more useful
FOCAL envelope. Thus raising the lower bound is not sufficient by itself: PFS
benefits only when the resulting FOCAL admissions improve the choices available
to the guided policy.

\subsection{RQ3: How Sensitive Is PFS to the Probability $p$?}

Across $p\in\{0.60,0.70,0.80\}$ and the seven tested bounds, PFS consistently
reduces expansions on commonly solved N-Puzzle and TSP instances, with median
setting-level ratios of 0.121--0.151 and 0.021--0.029. On Pancake, the ratios
are 1.222--1.629 and decrease as $p$ increases, because fewer selections use
the unhelpful OPEN-head branch. These paired ratios exclude failures; at
$p=0.70$, the paired Pancake ratio is 1.354, whereas the slightly higher PFS
success rate yields the all-run penalized ratio 0.874 in
Table~\ref{tab:headline}. Thus, the conclusions of RQ1 are stable across the
tested probabilities, but no universal $p$ emerges. A suitable value appears to
depend on the domain--heuristic pair and its observed $f_{\min}$-plateau
structure: smaller $p$ favors lower-bound advancement when FOCAL admission is
delayed, whereas larger $p$ favors already effective guided selection.

\ifdefined\nostatisticaltests
\subsection{RQ4: Does Probabilistic Scheduling Improve Anytime Search?}

We compare AFS, APFS, ADPS, and APDPS on 81 small and 234 medium metric GCTSP
instances. All four use the admissible quota-Kruskal-forest (QKF) lower bound;
AFS and APFS use prize deficit as their secondary FOCAL key, whereas ADPS and
APDPS use potential ordering with Envelope FOCAL. The anytime schedule starts
at $w_0=3$, updates $w\leftarrow C_{\mathrm{best}}/f_{\min}-0.10$, and stops
at $w\leq1.05$.

Figure~\ref{fig:anytime-cumulative-solutions} shows that APFS finds at least
one solution for 81/81 small and 189/234 medium instances, compared with 76/81
and 82/234 for AFS. Envelope FOCAL APDPS solves 80/81 small and 74/234 medium
instances, versus 78/81 and 25/234 for ADPS. Thus APDPS improves its
deterministic parent, especially on the medium set, while APFS retains the
highest first-solution coverage among all four methods.
\else
\subsection{RQ4: Does Probabilistic Scheduling Improve Anytime Search?}

We compare AFS, APFS, ADPS, and APDPS on 81 small and 234 medium metric GCTSP
instances. All four use the admissible quota-Kruskal-forest (QKF) lower bound;
AFS and APFS use prize deficit as their secondary FOCAL key, whereas ADPS and
APDPS use potential ordering. The anytime schedule starts at $w_0=3$, updates
$w\leftarrow C_{\mathrm{best}}/f_{\min}-0.10$, and stops at $w\leq1.05$. Runs
use $p=0.70$ for APFS and APDPS.
\fi

\begin{figure*}[t]
    \centering
    \includegraphics[width=0.84\textwidth]{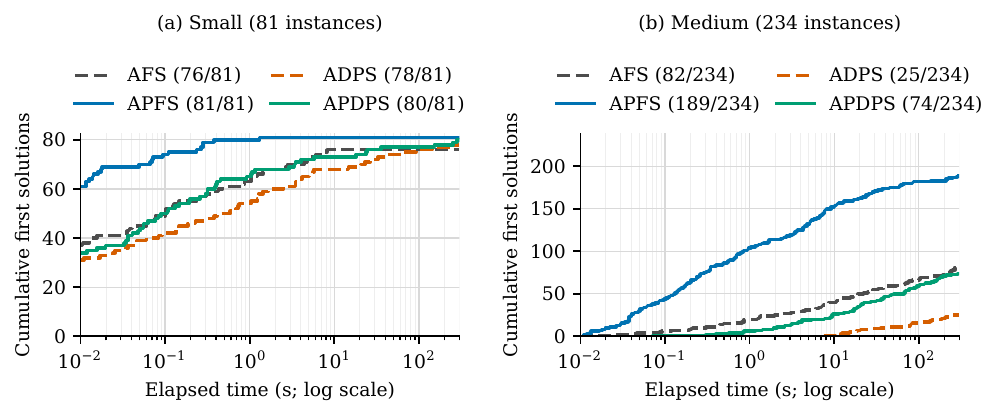}
    \caption{Cumulative number of matched GCTSP instances for which AFS, APFS,
    ADPS, or APDPS has found an incumbent by time $t$.}
    \label{fig:anytime-cumulative-solutions}

    \includegraphics[width=0.84\textwidth]{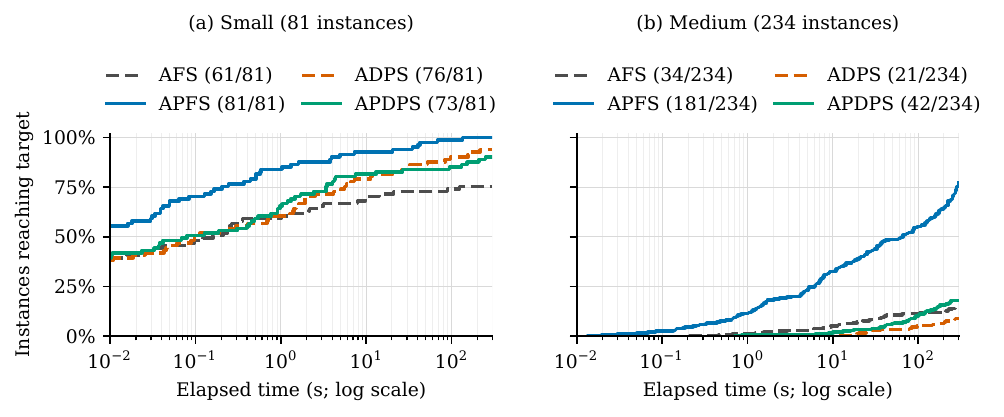}
    \caption{Time to reach a cost within 1\% of the lowest final value found by
    AFS, APFS, ADPS, or APDPS on each matched GCTSP instance.}
    \label{fig:anytime-target-attainment}
\end{figure*}

\ifdefined\nostatisticaltests
Figure~\ref{fig:anytime-target-attainment} evaluates subsequent solution
refinement. For each instance, the target is within 1\% of the lowest final
cost returned by any of the four evaluated methods. APFS reaches this target
on 81 small and 181 medium instances, compared with 61 and 34 for AFS.
ADPS/APDPS reach it on 76/73 small and 21/42 medium instances. APDPS therefore
doubles the medium target count of ADPS, although ADPS is slightly better on
the small set. On their common successes, APDPS/ADPS median expansion ratios
are 0.950 and 0.944, and runtime ratios are 0.513 and 0.735, for small and
medium instances, respectively.

The DPS-family result shows that probabilistic lower-bound advancement also
transfers to potential ordering when FOCAL admission limits progress. APFS
remains strongest because prize deficit explicitly directs newly eligible
states toward satisfying the residual quota, whereas dynamic potential does
not encode this feasibility signal.
\else
APFS/AFS solve 81/76 small and 189/82 medium instances
(Figure~\ref{fig:anytime-cumulative-solutions}).
Exact paired coverage tests are computed separately by dataset, with Holm
correction across the two tests. All five small-set discordances favor APFS
($q=0.0625$); all 107 medium-set discordances favor APFS
($q=2.5\times10^{-32}$).
\paragraph{Secondary DPS transfer.}
ADPS/APDPS solve 80/72 small and 59/43 medium instances.
All eight small discordances favor ADPS; on medium, 23 favor ADPS and seven
favor APDPS. Both dataset-level comparisons have Holm-adjusted
$q=0.0104$.

For Figure~\ref{fig:anytime-target-attainment}, $C_{\mathrm{ref},i}$ is the
lowest final cost returned by any of the four methods on instance $i$. AFS,
APFS, ADPS, and APDPS attain the 1\% target on 61, 81, 79, and 65 small
instances, and 35, 187, 47, and 22 medium instances. This common empirical
reference measures relative proximity and time, not an external optimality gap.

The FS family can order FOCAL by an independently chosen $h_2$; prize deficit
therefore directs AFS/APFS toward quota-feasible states. ADPS/APDPS instead
order nodes by dynamic potential and do not use this secondary key. At
$w_0=3$, the potential bound is already loose, so additional OPEN-head work
that raises $f_{\min}$ need not prioritize states that cover the remaining
prize. Consistent with this distinction, APDPS has lower coverage and target
attainment than ADPS. Its common-solved expansion ratios
are nevertheless 0.926 ($N=72$) and 0.114 ($N=36$), showing that lower
conditional effort does not offset the coverage loss.
\fi
 \section{Discussion and Conclusion}

PFS improves bounded-suboptimal focal search by following the guided FOCAL
policy with probability $p$ and expanding a minimum-$f$ OPEN node with
probability $1-p$. The latter branch can clear minimum-$f$ plateaus and raise
$f_{\min}$, so the larger threshold $w f_{\min}$ admits additional nodes for
guided selection while preserving the $w$-suboptimality guarantee.

Experiments on N-Puzzle, Pancake Sorting, and two TSP datasets show that this
mechanism improves search most clearly when delayed FOCAL admission is a
bottleneck. PFS substantially increases success and reduces expansions and
running time on N-Puzzle and TSP, while Pancake demonstrates the smaller gain
when the initial FOCAL is already sufficient. Envelope FOCAL PDPS improves the
all-run measures over DPS, although common-success results on Pancake and
looser TSP bounds show that OPEN-head selections add work when potential
guidance already succeeds. On anytime GCTSP, APFS solves 189/234 medium
instances versus 82/234 for AFS, while Envelope FOCAL APDPS solves 74/234
versus 25/234 for ADPS. APFS still leads all evaluated methods. Overall, active
$f_{\min}$ advancement improves focal search when FOCAL admission limits
progress and the guided policy can use the newly eligible states.
 
\appendix
\section{Appendix}

\subsection{Exact DPS Eligibility}
\label{app:dps-eligibility}

Under the main theorem's assumptions, an exact maximum-potential node lies in
$f(n)\leq w f_{\min}$. Let $m$ minimize $f$ in OPEN. For $h(m)>0$,
$u_w(m)=[w(g(m)+h(m))-g(m)]/h(m)=w+(w-1)g(m)/h(m)\geq1$.
Thus a maximum-potential $n$ satisfies $u_w(n)\geq1$, equivalently
$f(n)\leq w f_{\min}$. Exact DPS therefore need only consider FOCAL, which
contains a minimum-$f$ node whenever OPEN is nonempty and $w\geq1$.

\subsection{Admissible Quota--Kruskal-Forest Lower Bound}

For state $n$, let $x$ be its current vertex, $d$ the depot, $U(n)$ its covered
customers, and $R(n)=\max\{0,Q-|U(n)|\}$ the residual quota. Let $C_v(n)$ be
the uncovered customers covered by $v\notin\{x,d\}$ and sort the positive
capacities as $C_{(1)}\geq\cdots\geq C_{(m)}$. Set
$k(n)=\min\{j:\sum_{i=1}^{j}C_{(i)}(n)\geq R(n)\}$; every feasible completion
must visit at least $k(n)$ such vertices. For $R(n)>0$, form the complete
undirected graph on $\{x,d\}\cup\{v:C_v(n)>0\}$; QKF is the cost of its
cheapest $e(n)=k(n)+\mathbf{1}[x\neq d]$ acyclic edges, selected by Kruskal
\citep{kruskal1956shortest}. If $R(n)=0$, it is $0$ at $d$ and $c(x,d)$
otherwise. Every feasible completion contains such an acyclic subset no more
costly than itself; hence $h_{\mathrm{QKF}}(n)\leq h^*(n)$ and QKF is
admissible.

\ifdefined\nostatisticaltests\else
\subsection{Paired Statistical Procedures}

All comparisons match the parent and probabilistic method on the same instance
and configuration.  Coverage uses every pair.  If $b$ instances are solved
only by the parent and $c$ only by the probabilistic method, the exact
two-sided McNemar test evaluates
$X\sim\operatorname{Binomial}(b+c,1/2)$ \citep{mcnemar1947sampling}.  It tests whether discordant outcomes favor one
method more than expected under equal paired coverage.

Continuous metrics use instances solved by both methods.  Here
$y_i^{\mathrm{prob}}$ and $y_i^{\mathrm{parent}}$ are the values of the metric
being analyzed for matched instance $i$, obtained by the probabilistic method
and its deterministic parent, respectively.  We form paired ratios
$r_i=y_i^{\mathrm{prob}}/y_i^{\mathrm{parent}}$ and report their median and a
95\% percentile interval.  The paired bootstrap \citep{efron1979bootstrap}
resamples the $N$ observed pairs with replacement 10,000 times (seed 20260720)
and recomputes the median; it does not require 10,000 instances.

The two-sided Wilcoxon signed-rank test \citep{wilcoxon1945individual} applies
to $d_i=\log(y_i^{\mathrm{parent}}/y_i^{\mathrm{prob}})$: zero differences are
removed, $|d_i|$ are ranked, and positive and negative rank sums are compared.
It tests for a systematic paired shift without assuming normal metric values.
Holm's step-down procedure \citep{holm1979simple} adjusts the resulting
$p$-values within each parent--probabilistic family and metric to control the
family-wise error rate.  The 1\% win/tie/loss band is descriptive, not a test.
\fi
 \FloatBarrier
\bibliography{ref}
\end{document}